\documentclass[letterpaper, 10 pt, conference]{ieeeconf}

\usepackage{amsmath}
\usepackage{amssymb}
\usepackage{gensymb}
\usepackage{graphicx}
\usepackage[font=small]{caption}
\usepackage{breqn}
\usepackage{xcolor}
\usepackage{algorithm}
\usepackage{algpseudocode}
\usepackage{bm}
\newtheorem{lemma}{Lemma}

\usepackage{subcaption}
\usepackage[hidelinks]{hyperref}

\DeclareMathOperator*{\argmax}{arg\,max}
\newcommand{\KL}{\mathrm{KL}}
\newcommand{\TV}{\mathrm{TV}}
\newcommand{\E}{\mathbb{E}}

\IEEEoverridecommandlockouts  
\title{\LARGE \bf
Data-Efficient Hierarchical Goal-Conditioned Reinforcement Learning via Normalizing Flows
}

\author{Shaswat Garg$^{1}$, Matin Moezzi$^{1}$ and Brandon Da Silva$^{1}$
\thanks{Shaswat Garg,Matin Moezzi and Brandon Da Silva are with ArenaX Labs. . GitHub repository: \texttt{https://github.com/Shaswat2001/heirarchical\_RL}}
}

\begin{document}
\maketitle

\begin{abstract}
Hierarchical goal-conditioned reinforcement learning (H-GCRL) provides a powerful framework for tackling complex, long-horizon tasks by decomposing them into structured subgoals. However, its practical adoption is hindered by poor data efficiency and limited policy expressivity, especially in offline or data-scarce regimes. In this work, Normalizing flow-based hierarchical implicit Q-learning (NF-HIQL), a novel framework that replaces unimodal gaussian policies with expressive normalizing flow policies at both the high- and low-levels of the hierarchy is introduced. This design enables tractable log-likelihood computation, efficient sampling, and the ability to model rich multimodal behaviors. New theoretical guarantees are derived, including explicit KL-divergence bounds for Real-valued non-volume preserving (RealNVP) policies and PAC-style sample efficiency results, showing that NF-HIQL preserves stability while improving generalization. Empirically, NF-HIQL is evaluted across diverse long-horizon tasks in locomotion, ball-dribbling, and multi-step manipulation from OGBench. NF-HIQL consistently outperforms prior goal-conditioned and hierarchical baselines, demonstrating superior robustness under limited data and highlighting the potential of flow-based architectures for scalable, data-efficient hierarchical reinforcement learning.
\end{abstract}

\section{INTRODUCTION}

Generalization is a cornerstone of Reinforcement learning (RL): it empowers agents to tackle new, unseen tasks and adapt to ever-changing environments \cite{wu2024unifying}. Despite remarkable progress—deep RL agents now master continuous control challenges like locomotion \cite{margolis2024rapid}, dexterous object manipulation \cite{lin2025sim}, and robotic arm coordination \cite{hao2024coordinated} — these successes have mostly focused on short-horizon, low-level motor skills executed in isolation. Real-world problems, by contrast, demand hierarchical reasoning: agents must integrate perception, planning, and decision-making across multiple layers of abstraction and compose primitive behaviors into long-horizon strategies. H-GCRL provides a natural framework for tackling complex, long-horizon tasks by decomposing them into sequences of subgoals, each managed by a corresponding low-level policy. This hierarchical structure allows agents to reason over multiple timescales and compose simpler behaviors into more sophisticated strategies. Recent advances in offline Goal-conditioned RL (GCRL) \cite{gong2024goal} and hierarchical extensions \cite{li2022hierarchical} demonstrate that it is possible to train such agents using large, unlabeled datasets—such as videos or multi-task demonstrations—by conditioning policies on desired goals, even when explicit reward or action information is missing.

However, collecting such large and diverse datasets, especially those that cover a wide range of goals and transitions, is often infeasible in practice due to high data acquisition costs, safety concerns, and limitations in real-world deployment. This makes it essential to develop algorithms that can generalize effectively from limited or weakly supervised data, ideally leveraging structure, compositionality, and inductive biases inherent in the task space.

To tackle the challenge of sample inefficiency \cite{blonde2019sample}, recent approaches have primarily focused on leveraging powerful generative models, which also provide the side benefit of greater policy expressiveness. Diffusion models offer rich expressiveness and have shown promise in capturing complex action distributions, but they are computationally expensive due to the need to solve differential equations during training and inference \cite{mei2025deep, cao2024survey}. Autoregressive models scale more efficiently and allow parallel training, yet they often rely on learning discrete representations of actions, which can introduce quantization artifacts and complicate optimization \cite{xiong2024autoregressive}. In contrast, gaussian policies are lightweight and efficient to train but lack the capacity to represent multimodal or structured behaviors, limiting their effectiveness in hierarchical or goal-conditioned settings \cite{choi2024data}.

In this work, a novel approach is proposed that leverages Normalizing flows (NFs)—specifically the RealNVP architecture—within the Hierarchical implicit Q-learning (HIQL) framework to bridge this gap \cite{park2023hiql}. NFs strike a favorable balance between expressivity and computational tractability \cite{kobyzev2020normalizing}. In particular, they provide exact likelihood evaluation, which yields unbiased and lower-variance gradient estimates during training, leading to more stable optimization and improved sample efficiency compared to methods that rely on approximate likelihoods such as MCMC-based or variational approaches (e.g., energy-based models), a property that is especially advantageous in offline and data-constrained regimes \cite{dinh2016density}. By integrating normalizing flows into both the high-level and low-level policies of a hierarchical framework, a more expressive and data-efficient alternative is introduced to traditional policy representations. Furthermore, a theoretical analysis is presented that the learned policies are bounded in KL divergence and enjoy provable guarantees on sample efficiency, pushing the boundaries of scalable and robust long-horizon decision-making.

\section{RELATED WORK}

This section provides an overview of recent advancements in HGCRL and GCRL, with a particular focus on efforts to improve sample efficiency—especially through the use of generative models.

GCRL enables agents to learn a spectrum of tasks by conditioning on a goal input, fostering generalization across different outcomes \cite{gong2024goal}. A major challenge, however, is sample inefficiency in sparse-reward settings. Relabeling strategies like Hindsight experience replay (HER) \cite{andrychowicz2017hindsight} and density-based goal sampling \cite{yang2021density} address this by reusing or prioritizing goals. Generative approaches also improve efficiency, e.g., learning latent dynamics models for planning \cite{nair2020goal}, synthesizing goal-directed rollouts with GANs \cite{charlesworth2020plangan}, or incorporating planning into offline GCRL \cite{zhu2021mapgo}. While these methods generate additional data or plans, they rely on accurate learned models or GAN training, which can be brittle and hard to scale. In contrast, the proposed work sidesteps explicit planning by leveraging expressive policy models and hierarchical value-based learning to improve efficiency directly.

Long-horizon goal-reaching tasks benefit from hierarchical decomposition, where high-level subgoals improve learning efficiency \cite{kulkarni2016hierarchical, nachum2018near}. HIQL \cite{park2023hiql} extends this by learning a single goal-conditioned value function offline and deriving both high- and low-level policies, with subgoals proposed in latent space. This provides clearer learning signals and outperforms prior offline GCRL methods. Theory also supports hierarchy as a way to reduce sample complexity \cite{robert2023sample}. Yet, most HGCRL methods—including HIQL—rely on simple Gaussian policies, limiting their ability to capture complex, multimodal behaviors, and neglect modern generative models. The proposed method addresses this by using normalizing flows, enabling richer hierarchical policies while retaining tractable training and efficient sampling.

Recent advances in generative modeling for RL have introduced more expressive policy classes to improve sample efficiency. Diffusion models, for instance, have been applied to goal-conditioned settings: \cite{jain2023learning} proposed a diffusion-based policy that achieves strong offline performance via denoising-based inference. Generative Flow Networks (GFlowNets) offer trajectory-level diversity, as in Goal2FlowNets \cite{madangoal2flownet}, which enhance generalization. However, both diffusion models and GFlowNets involve complex training and sampling, leading to high computational costs. Normalizing flows (NFs) provide a tractable alternative \cite{ghugare2025normalizing}; SAC-NF \cite{mazoure2020leveraging} demonstrated improved convergence and expressivity by replacing Gaussian policies with NFs.

However, existing applications of NFs have primarily focused on flat policy architectures. The proposed work bridges this gap by integrating NFs into both the high-level and low-level policies of the HIQL framework. This yields expressive, multimodal policies at each level of hierarchy while preserving the tractable training and efficient sampling that NFs offer. In doing so, the method introduced enhances both the generalization capacity and sample efficiency of HGCRL in complex, long-horizon environments.

\section{BACKGROUND}

The problem is framed as a Markov decision process (MDP) \cite{bellman1957markovian} and a dataset $\mathcal{D}$, defined by a tuple $<S,A,\mu,P,R, \gamma>$, where $S$ represents the set of possible states, $A$ represents the set of possible actions, $\mu \in \mathcal{P}(\mathcal{S})$ denotes an initial state distribution, $P: S \times A \rightarrow S$ is the state transition function that represents the conditional probability $P(\mathbf{s'}|\mathbf{s},\mathbf{a})$ or deterministic \textcolor{black}{function} $\mathbf{s'} = P(\mathbf{s},\mathbf{a})$, and $R: S \times G \rightarrow \mathbb{R}$ represents the goal-conditioned reward function. The dataset $\mathcal{D}$ consists of trajectories $\tau = (s_0, a_0, s_1, a_1, \ldots, s_T)$. It is assumed that the goal space $\mathcal{G}$ is identical to the state space, that is, $\mathcal{G} = \mathcal{S}$. The aim is to learn a goal-conditioned policy $\pi(a \mid s, g)$ using $\mathcal{D}$ such that the expected cumulative reward

\begin{equation}
J(\pi) = \mathbb{E}_{g \sim p(g),\, \tau \sim p^\pi(\tau)} \left[\sum_{t=0}^T \gamma^t R(s_t, g)\right]
\end{equation}

is maximized. The trajectory distribution under policy $\pi$ is given by

\begin{equation}
p^\pi(\tau) = \mu(s_0) \prod_{t=0}^{T-1} \pi(a_t \mid s_t, g)\, P(s_{t+1} \mid s_t, a_t),
\end{equation}

where $\gamma$ denotes the discount factor and $p(g)$ is the distribution over goals.

\section{ALGORITHM}

We consider a HGCRL setup following HIQL \cite{park2023hiql}, where a single value function $V(s,g)$ guides two policies: a high-level subgoal policy $\pi^h$ and a low-level action policy $\pi^\ell$. The high-level policy proposes a future state (latent subgoal) $s_{t+k}$ given the current state $s_t$ and goal $g$, while the low-level policy selects actions $a_t$ to reach it. In NF-HIQL, both policies are modeled as conditional normalizing flows, enabling exact density, gradient, and entropy computation while supporting expressive, multimodal behaviors.

\subsection{Flow-Based Policy Parameterization}

Concretely as shown in Algorithm \ref{alg:1}, each policy is defined by an invertible function that maps a noise vector to an output (subgoal or action). For example, the high-level policy is given as follows:

\begin{equation}
s_{t+k} = f_H(u;\,s_t,g), \quad u \sim \mathcal{N}(0,I),
\end{equation}

where $f_H(\cdot;s_t,g)$ is a neural-network flow conditioned on $(s_t,g)$.  In other words, latent noise $u$ is sampled from a standard Gaussian and passed through $f_H$ to produce a candidate subgoal $s_{t+k}$.  Similarly, the low-level policy uses its own flow $f_\ell$ to map $v\sim \mathcal{N}(0,I)$ (conditioned on the current state and chosen subgoal) to an action: $a_t = f_\ell(v;s_t,s_{t+k})$.

Because $f_H$ and $f_\ell$ are bijective with tractable Jacobians, the log-density of an output can be computed exactly via the change-of-variables formula.  For instance, if $u = f_H^{-1}(s_{t+k};s_t,g)$, then

\small
\begin{dmath}
\log \pi^h(s_{t+k}\mid s_t,g) 
= \log p_H(u) - \log\Bigl|\det\bigl(\frac{\partial f_H(u;\,s_t,g)}{\partial u}\bigr)\Bigr|.
\end{dmath}

\normalsize
Here $p_H(u)$ is the Gaussian base density (e.g.\ $\mathcal{N}(0,I)$), and $\det(\partial f_H/\partial u)$ is the Jacobian determinant of the flow.  An analogous formula holds for the low-level policy: if $v = f_\ell^{-1}(a_t;,s_t,s_{t+k})$, then

\small
\begin{dmath}
\log \pi^\ell(a_t\mid s_t,s_{t+k}) 
= \log p_\ell(v) - \\ \log\Bigl|\det\bigl(\frac{\partial f_\ell(v;\,s_t,s_{t+k})}{\partial v}\bigr)\Bigr|.
\end{dmath}

\normalsize
In short, flow transforms endow the policy with an analytic log-probability while providing high expressivity. By stacking invertible layers, a simple base density is transformed into a richer, potentially multimodal distribution. In this work, RealNVP \cite{dinh2016density} serves as a universal approximator for continuous densities, enabling policies that capture complex multimodal action or subgoal distributions while supporting efficient sampling and exact likelihood evaluation.

\subsection{Log-Probability, Entropy, and Advantage-Weighted Objectives}

Because the flow policies admit exact densities, the usual advantage-weighted learning objectives can be written in the closed form.  As in HIQL, the high-level advantage is defined as $A^h(s_t,s_{t+k},g) = V(s_{t+k},g) - V(s_t,g)$ and the low-level advantage as $A^\ell(s_t,a_t,s_{t+1},s_{t+k}) = V(s_{t+1},s_{t+k}) - V(s_t,s_{t+k})$.  Then the weighted maximum-likelihood (AWR-style) objectives are:

\begin{itemize}
    \item High-level objective: 
    \begin{equation}
        J^h(\theta_H) = \mathbb{E}_{\text{data}}\Bigl[e^{\beta A^h}\,\log\pi^h_{\theta_H}(s_{t+k}\mid s_t,g)\Bigr].
    \end{equation}
    \item Low-level objective: 
    \begin{equation}
        J^\ell(\theta_L) = \mathbb{E}_{\text{data}}\Bigl[e^{\beta A^\ell}\,\log\pi^\ell_{\theta_L}(a_t\mid s_t,s_{t+k})\Bigr].
    \end{equation}
\end{itemize}

Each expectation is over logged offline transitions (with subgoals $s_{t+k}$ and actions $a_t$) with weight $\exp(\beta A)$, so that higher-advantage samples are upweighted.  Substituting the flow log-densities above makes both $J^h$ and $J^\ell$ fully differentiable functions of the flow parameters $\theta_H,\theta_L$.  In particular, the gradients take the simple form as follows:

\begin{equation}
\begin{split}
    \nabla_{\theta_H}J^h &= \mathbb{E}\bigl[e^{\beta A^h}\,\nabla_{\theta_H}\log\pi^h(s_{t+k}\mid s_t,g)\bigr], 
\\
    \nabla_{\theta_L}J^\ell &= \mathbb{E}\bigl[e^{\beta A^\ell}\,\nabla_{\theta_L}\log\pi^\ell(a_t\mid s_t,s_{t+k})\bigr].
\end{split}
\end{equation}

Since $\log\pi^h$ and $\log\pi^\ell$ are given in closed form by the flow (the only trainable part of $\log p(u)$ is constant), then  $\nabla_\theta \log\pi = -\nabla_\theta\bigl[\log|\det(\partial f/\partial u)|\bigr]$.  Thus no policy sampling or likelihood-ratio estimators are needed: the Jacobian log-det gradient can be computed analytically for each data point.

Because the flows yield exact densities, the policy entropies can also be computed in closed form if desired.  For example,

\begin{dmath}
H(\pi^h) = -\mathbb{E}_{s\sim\pi^h}[\log\pi^h(s)]
= -\mathbb{E}_{u\sim p_H}\bigl[\log p_H(u) - \log|\det(\partial f_H/\partial u)|\bigr],
\end{dmath}

which can be estimated by sampling $u\sim p_H$ (all terms inside are known).  In short, all key quantities — log-likelihoods and entropies — are tractable and exact for the flow policies.


\begin{algorithm}[t]
\caption{Offline HIQL with Normalizing Flow Policies}
\label{alg:1}
\begin{algorithmic}[1]
\Require Dataset $\mathcal{D}$, Networks: value function $V_{\theta_V}(s, g)$, target value function $V_{\bar{\theta}_V}(s, g)$, high-level policy $\pi^h_{\theta_H}(s_{t+k}| s_t, g)$, low-level policy $\pi^\ell_{\theta_L}(a_t | s_t, s_{t+k})$, Hyper parameters: learning rates $\alpha_V$, $\alpha_H$, $\alpha_L$; temperature $\beta$

\While{not converged}
    \State \textcolor{gray}{// \textbf{1. Update Value Function using action-free IQL}}
    \State Sample $(s_t, s_{t+1}, g)$ from $\mathcal{D}$
    \small
    \State $y \gets r(s_t, g) + \gamma V_{\bar{\theta}_V}(s_{t+1}, g)$
    \State $\theta_V \gets \theta_V - \alpha_V \nabla_{\theta_V} \rho_\tau\left(y - V_{\theta_V}(s_t, g)\right)$
    \normalsize
    \State \textcolor{gray}{// \textbf{2. Update High-Level Policy (Normalizing Flow)}}
    \State Sample $(s_t, s_{t+k}, g)$ from $\mathcal{D}$
    \small
    \State $A^h \gets V_{\theta_V}(s_{t+k}, g) - V_{\theta_V}(s_t, g)$
    \normalsize
    \State Compute $\log \pi^h_{\theta_H}(s_{t+k} \mid s_t, g)$ via flow:
    \small
    \Statex \hspace{0.7cm} $u = f_H^{-1}(s_{t+k}; s_t, g)$
    \Statex \hspace{0.7cm} $\log \pi^h = \log p_H(u) - \log |\det \partial f_H / \partial u|$
    \State $\theta_H \gets \theta_H + \alpha_H \nabla_{\theta_H} \left[ e^{\beta A^h} \log \pi^h \right]$
    \normalsize
    \State \textcolor{gray}{// \textbf{3. Update Low-Level Policy (Normalizing Flow)}}
    \State Sample $(s_t, a_t, s_{t+1}, s_{t+k})$ from $\mathcal{D}$
    \small
    \State $A^\ell \gets V_{\theta_V}(s_{t+1}, s_{t+k}) - V_{\theta_V}(s_t, s_{t+k})$
    \normalsize
    \State Compute $\log \pi^\ell_{\theta_L}(a_t \mid s_t, s_{t+k})$ via flow:
    \small
    \Statex \hspace{0.7cm} $v = f_\ell^{-1}(a_t; s_t, s_{t+k})$
    \Statex \hspace{0.7cm} $\log \pi^\ell = \log p_\ell(v) - \log |\det \partial f_\ell / \partial v|$
    \State $\theta_L \gets \theta_L + \alpha_L \nabla_{\theta_L} \left[ e^{\beta A^\ell} \log \pi^\ell \right]$
    \normalsize
\EndWhile
\end{algorithmic}
\end{algorithm}


\begin{figure*}
        \centering
     \includegraphics[scale=0.7]{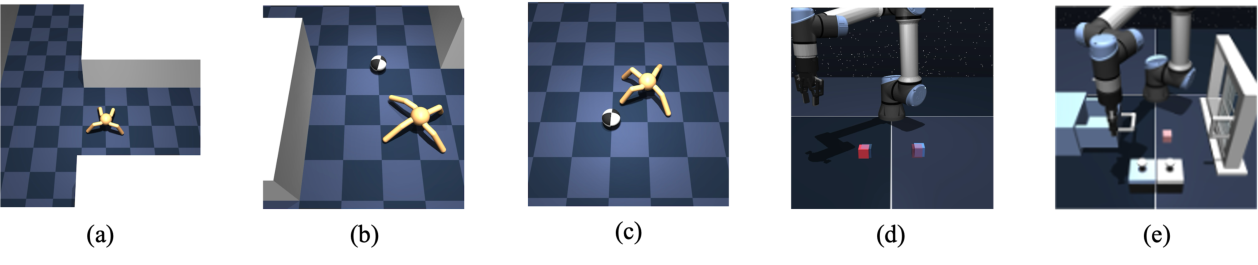}
    \caption{Evaluation environments: (a) AntMaze—medium-navigate (long-horizon maze navigation); (b) AntSoccer—medium-navigate (wall-bounded dribbling and navigation); (c) AntSoccer—arena-navigate (open-field dribbling and navigation); (d) Cube—single-play (pick-and-place from play data); (e) Scene—play (multi-object, multi-step sequencing from play) \cite{park2024ogbench}.}
    \label{fig:env_fig}
    \vspace{-0.5cm}
\end{figure*}


\subsection{Theoretical Guarantees}

To complement the algorithmic design, NF-HIQL is supported by theoretical results on stability and efficiency. First, an upper bound on the KL divergence between the hierarchical policy and the behavior policy in $\mathcal{D}$ ensures the learned policy stays close to the data distribution, mitigating out-of-distribution actions and extrapolation error. Second, a PAC-style sample efficiency bound shows that the hierarchical policies remain near-optimal under finite data. Together, these results establish NF-HIQL as both stable and provably sample-efficient, reinforcing its practicality for real-world settings with limited data. Full technical proofs are provided in Appendix \ref{appdx:1}.
\setcounter{lemma}{1}
\begin{lemma}[KL Divergence Bound]
Let $\pi^b(\cdot \mid s)$ denote the behavior policy and $\pi_\theta(\cdot \mid s)$ the learned RealNVP policy given state $s$. 
If the action space is bounded and the behavior density is capped by a constant $M<\infty$, then there exists a constant $B<\infty$ (determined by the RealNVP architecture) such that
\begin{equation}
\mathrm{KL}\!\big(\pi^b(\cdot\mid s)\,\|\,\pi_\theta(\cdot\mid s)\big) \;\le\; B+\log M.
\end{equation}
\end{lemma}
\begin{proof}
See Appendix~\ref{appdx:1}.
\end{proof}

\begin{lemma}[PAC-Style Sample Efficiency]
Let $\hat\pi_h,\hat\pi_\ell$ be the policies learned by advantage-weighted Maximum Likelihood Estimation
(MLE). With probability at least $1-\delta$,
\small
\begin{dmath}
J(\pi^\star)-J(\hat\pi_{h,\ell})
\le
\frac{H_h A_{\max,h}}{1-\gamma}\,
\sqrt{\frac{C_h}{2}\;
\Big(4\,\mathfrak R_{n_h}(\mathcal F_h)+2B_h\sqrt{\tfrac{\log(2/\delta)}{2n_h}}\Big)}
\\ +
\frac{H_\ell A_{\max,\ell}}{1-\gamma}\,
\sqrt{\frac{C_\ell}{2}\;
\Big(4\,\mathfrak R_{n_\ell}(\mathcal F_\ell)+2B_\ell\sqrt{\tfrac{\log(2/\delta)}{2n_\ell}}\Big)}
\\ + 
\varepsilon_V,
\end{dmath}
\normalsize
\end{lemma}

\begin{proof}

See Appendix~\ref{appendix:sample}.
\end{proof}


\begin{table*}
\centering
\footnotesize
\begin{tabular}{lcccccc}
\hline
Environment &
BESO & GCIQL & CRL & HIQL & NF-GCIQL & NF-HIQL \\
\hline
\texttt{antmaze-medium-navigate}  & 85{\tiny $\;\pm 7$} & 71{\tiny $\;\pm 4$} &  95{\tiny $\;\pm 1$} &  \textbf{96}{\tiny $\;\pm 1$} & 82{\tiny $\;\pm 3$} & 95{\tiny $\;\pm 2$} \\
\texttt{antsoccer-medium-navigate}    & 12{\tiny $\;\pm 3$} & 7{\tiny $\;\pm 1$} & 3{\tiny $\;\pm 1$} & 13{\tiny $\;\pm 1$} & 6{\tiny $\;\pm 4$} & \textbf{14}{\tiny $\;\pm 2$}  \\
\texttt{antsoccer-arena-navigate} & 56{\tiny $\;\pm 2$} & 50{\tiny $\;\pm 2$} & 23{\tiny $\;\pm 2$} & 58{\tiny $\;\pm 2$} & 30{\tiny $\;\pm 3$} & \textbf{73}{\tiny $\;\pm 1$}  \\
\texttt{cube-single-play}  & 21{\tiny $\;\pm 2$} & 68{\tiny $\;\pm 6$} & 19{\tiny $\;\pm 2$} & 15{\tiny $\;\pm 3$} & \textbf{70}{\tiny $\;\pm 1$} & 37{\tiny $\;\pm 2$}  \\
\texttt{scene-play}  & 81{\tiny $\;\pm 3$} & 51{\tiny $\;\pm 4$} & 19{\tiny $\;\pm 2$} & 38{\tiny $\;\pm 3$} & 50{\tiny $\;\pm 2$} & \textbf{40}{\tiny $\;\pm 3$}  \\
\hline
\end{tabular}
\caption{Overall success rate (\%) across all the tasks — dataset size: \textbf{100\%} of the available offline dataset.}
\label{tab:results-100}
\end{table*}

\begin{table*}
\centering
\footnotesize
\begin{tabular}{lcccccc}
\hline
Environment &
BESO & GCIQL & CRL & HIQL & NF-GCIQL & NF-HIQL \\
\hline
\texttt{antmaze-medium-navigate} & 63{\tiny $\;\pm 6$} &  24{\tiny $\;\pm 2$} & 50{\tiny $\;\pm 2$} & 58{\tiny $\;\pm 4$} & 64{\tiny $\;\pm 3$} &  \textbf{72}{\tiny $\;\pm 4$}\\
\texttt{antsoccer-medium-navigate} & 1{\tiny $\;\pm 1$} & 0{\tiny $\;\pm 0$} & 0{\tiny $\;\pm 0$} & 0{\tiny $\;\pm 0$} & \textbf{7}{\tiny $\;\pm 2$} & 3{\tiny $\;\pm 2$} \\
\texttt{antsoccer-arena-navigate}  &  30{\tiny $\;\pm 2$} & 2{\tiny $\;\pm 1$} & 0{\tiny $\;\pm 0$} & 1{\tiny $\;\pm 1$} & 41{\tiny $\;\pm 4$} &  \textbf{73}{\tiny $\;\pm 4$}\\
\texttt{cube-single-play} & 4{\tiny $\;\pm 1$} &  10{\tiny $\;\pm 6$} & 6{\tiny $\;\pm 3$} & 4{\tiny $\;\pm 2$} & \textbf{40}{\tiny $\;\pm 9$} &  36{\tiny $\;\pm 4$}\\
\texttt{scene-play} & 14{\tiny $\;\pm 2$}  &  8{\tiny $\;\pm 3$} & 2{\tiny $\;\pm 2$} & 6{\tiny $\;\pm 4$} & 33{\tiny $\;\pm 4$} &  \textbf{36}{\tiny $\;\pm 3$}\\
\hline
\end{tabular}
\caption{Overall success rate (\%) across all the tasks — dataset size: \textbf{50\%} of the available offline dataset.}
\label{tab:results-50}
\end{table*}

\normalsize
\begin{figure*}
        \centering
     \includegraphics[scale=0.48]{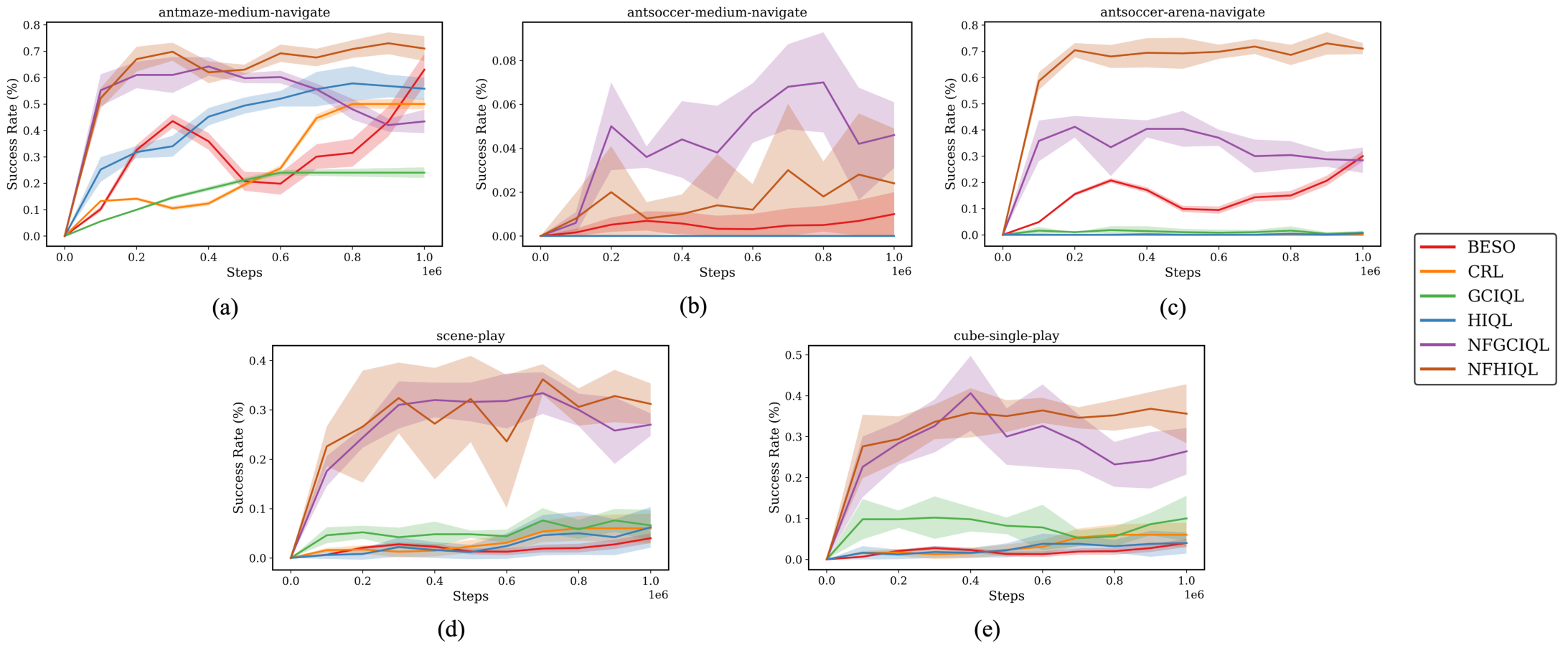}
    \caption{Success rate (\%) across training steps on OGBench environments. NF-HIQL consistently outperforms baselines, showing faster convergence and higher final success rates, particularly in complex manipulation tasks (\texttt{cube-single-play}, \texttt{scene-play}) and multi-agent soccer settings.}
    \label{fig:results_fig}
    \vspace{-0.5cm}
\end{figure*}

\section{RESULTS}

NF-HIIQL is evaluated on five OGBench tasks spanning long-horizon locomotion, ball-dribbling, and multi-step manipulation: antmaze-medium-navigate, antsoccer-medium-navigate, antsoccer-arena-navigate, cube-single-play, and scene-play (Figure \ref{fig:env_fig}). We follow OGBench’s official environment definitions, dataset splits, and multi-goal success-rate protocol for offline goal-conditioned RL \cite{park2024ogbench}. Each algorithm is trained on the same offline dataset for 1M transitions with five random seeds on NVIDIA T4 GPUs. Unless noted otherwise, we use the state-based benchmark variants, adapted to the benchmark’s evaluation goals (five per manipulation environment) when reporting success rates and confidence intervals.

Comparisons include three representative offline GCRL baselines from OGBench—GCIQL, CRL, and HIQL—along with the diffusion-based BESO \cite{reuss2023goal}, providing a strong reference set across navigation and manipulation. Two additional flow-based variants are considered: NF-HIQL (hierarchical flow policies at both levels) and NF-GCIQL (a flow policy under the GCIQL objective). Consistent with the motivation for expressive, multimodal policies, NF-HIQL outperforms all baselines.

To assess sample efficiency, each experiment is conducted in two regimes: (i) training on 100\% of the available dataset and (ii) training on exactly 50\% of the same dataset (a uniform halving of trajectories). These paired settings let us quantify the sample-efficiency gap between flow-based methods and unimodal counterparts.

NF-HIQL’s success rates under both full and reduced dataset regimes clearly demonstrate its superior sample efficiency, as reported in Table~\ref{tab:results-100} and Table~\ref{tab:results-50}. With $100\%$ of the available dataset, NF-HIQL achieves success rates that are competitive with or stronger than the best OGBench baselines. For example, in \texttt{antmaze-medium-navigate}, NF-HIQL achieves a success rate of $95 \tiny{\pm 2}\%$, nearly identical to HIQL’s $96 \tiny{\pm 1}\%$ from \cite{park2024ogbench}, while outperforming GCIQL ($71 \tiny{\pm 4}\%$), NF-GCIQL ($82 \tiny{\pm 3}\%$), and the diffusion-based BESO ($85 \tiny{\pm 7}\%$). Similarly, in \texttt{antsoccer-medium-navigate}, NF-HIQL reaches $14 \tiny{\pm 2}\%$ success rate, which is over $2.3\times$ higher than GCIQL ($7 \tiny{\pm 1}\%$), more than double the NF-GCIQL's ($6 \tiny{\pm 4}\%$), and also outperforms BESO ($12 \tiny{\pm 3}\%$). The gains are even more substantial in \texttt{antsoccer-arena-navigate}, where NF-HIQL attains $73 \tiny{\pm 1}\%$ success rate, compared to HIQL’s $58 \tiny{\pm 2}\%$, NF-GCIQL’s $30 \tiny{\pm 3}\%$, and BESO’s $56 \tiny{\pm 2}\%$, yielding relative improvements of approximately $26\%$, $2.4\times$, and $30\%$ respectively. For manipulation tasks, NF-HIQL continues to dominate: in \texttt{cube-single-play}, NF-HIQL achieves a success rate of $36 \tiny{\pm 4}\%$, maintaining strong performance even with half the dataset. This represents a nearly $9\times$ improvement over HIQL ($4 \tiny{\pm 2}\%$) and BESO ($4 \tiny{\pm 1}\%$), and a $6\times$ gain over CRL ($6 \tiny{\pm 3}\%$). NF-HIQL also outperforms GCIQL ($10 \tiny{\pm 6}\%$) by more than $3.5\times$, and remains competitive with NF-GCIQL ($40 \tiny{\pm 9}\%$), trailing by only $10\%$ while offering substantially lower variance. These results highlight NF-HIQL’s robustness and efficiency in complex manipulation tasks under limited data.

The crucial distinction emerges when the dataset is reduced by half. With only $50\%$ of the data (Table~\ref{tab:results-50}), NF-HIQL still maintains strong success rates while other algorithms degrade sharply. On \texttt{antmaze-medium-navigate}, NF-HIQL achieves $72 \tiny{\pm 4}\%$ success rate, whereas HIQL drops to $58 \tiny{\pm 4}\%$, BESO falls further to $63 \tiny{\pm 6}\%$, and NF-GCIQL declines to $64 \tiny{\pm 3}\%$. This corresponds to NF-HIQL being $24\%$ stronger than HIQL, $14\%$ stronger than BESO, and $12.5\%$ stronger than NF-GCIQL under the reduced setting. The efficiency advantage becomes dramatic in \texttt{antsoccer-arena-navigate}, where NF-HIQL records the success rate of $73 \tiny{\pm 4}\%$ compared to HIQL’s $1 \tiny{\pm 1}\%$, BESO’s $30 \tiny{\pm 2}\%$, and NF-GCIQL’s $41 \tiny{\pm 4}\%$. In this case, NF-HIQL achieves a success rate that is $70 \times$ higher than HIQL, $2.4\times$ higher than BESO, and $1.8\times$ higher than NF-GCIQL, while nearly matching its own $100\%$-data performance.  

In \texttt{scene-play}, NF-HIQL achieves $36 \tiny{\pm 3}\%$ success, outperforming HIQL ($6 \tiny{\pm 4}\%$), BESO ($14 \tiny{\pm 2}\%$), and NF-GCIQL ($33 \tiny{\pm 4}\%$)—over $6\times$ stronger than HIQL, $2.5\times$ stronger than BESO, and $10\%$ stronger than NF-GCIQL. A similar trend holds in \texttt{cube-single-play}, where NF-HIQL records $35 \tiny{\pm 7}\%$, nearly matching its full-data result ($37 \tiny{\pm 2}\%$). In contrast, HIQL drops to $6 \tiny{\pm 3}\%$, BESO to $4 \tiny{\pm 1}\%$, and NF-GCIQL to $27 \tiny{\pm 5}\%$. Thus, NF-HIQL is $6\times$ stronger than HIQL, $9\times$ stronger than BESO, and $30\%$ stronger than NF-GCIQL with only half the data.

Even in the particularly challenging \texttt{antsoccer-medium-navigate}, where absolute success rates are low, NF-HIQL achieves $3 \tiny{\pm 2}\%$ with half the data, maintaining a measurable level of success compared to HIQL and CRL (both $0\%$), while BESO reaches only $1 \tiny{\pm 1}\%$. Thus, NF-HIQL still outperforms BESO by $3\times$ in relative success rate in this domain.  

These results confirm that NF-HIQL is not only a top-performing method in absolute terms but also markedly more sample-efficient. Detailed results are provided in the supplementary video. Across tasks, NF-HIQL trained on $50\%$ of the data often matches or exceeds HIQL, BESO, and GCIQL trained on the full dataset. Gains are most pronounced in multimodal tasks such as \texttt{antsoccer-arena-navigate}, \texttt{antmaze-medium-navigate}, and \texttt{scene-play}, where NF-HIQL achieves multiples higher success rates than both diffusion-based (BESO) and unimodal (HIQL, GCIQL) baselines. Overall, NF-HIQL emerges as the most robust and sample-efficient algorithm across the diverse OGBench domains.

\section{EXPERIMENTS}

\normalsize
\begin{figure}
        \centering
     \includegraphics[scale=0.3]{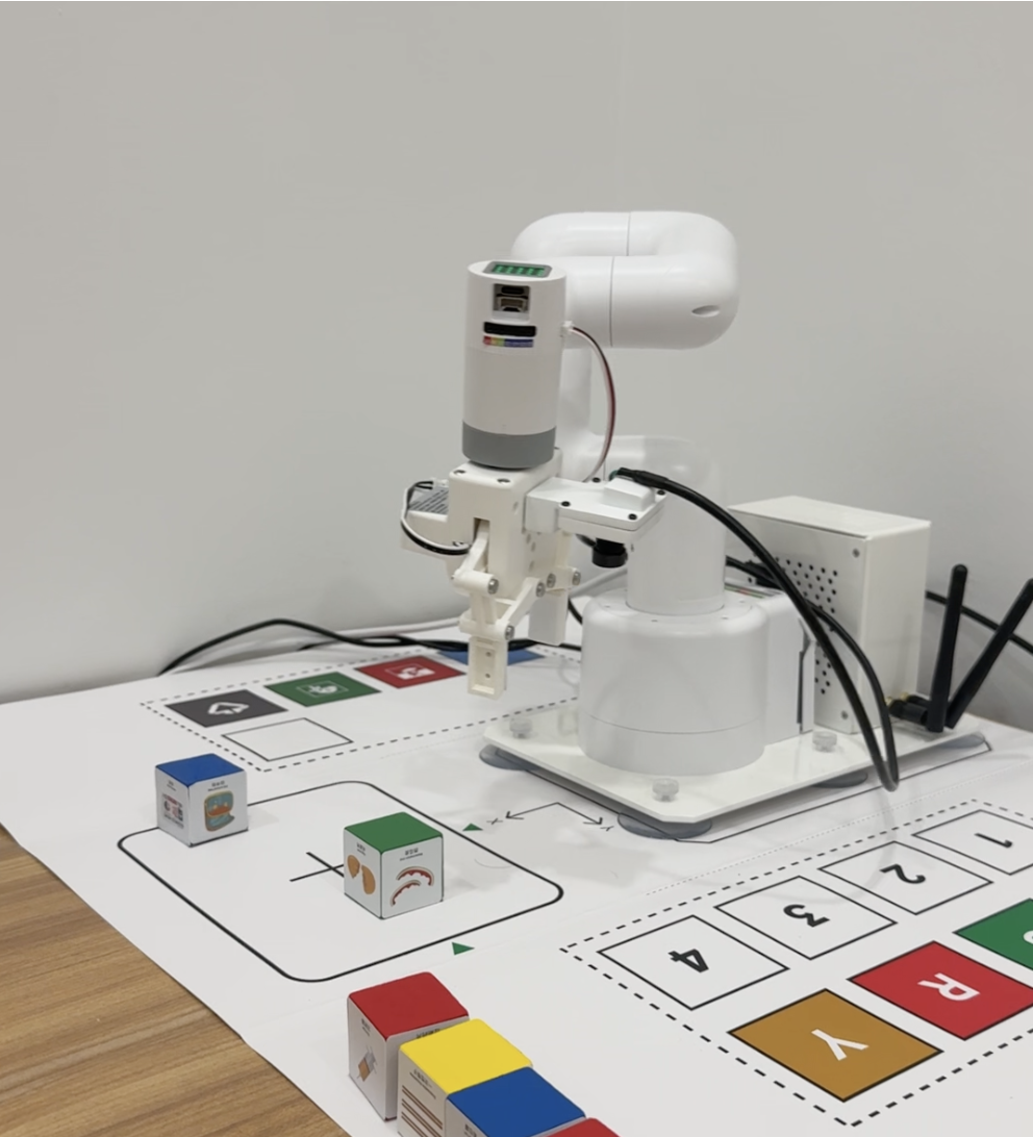}
    \caption{Experimental setup with the 6 DOF myCobot 280 arm.}
    \label{fig:real_robot}
    \vspace{-0.5cm}
\end{figure}

To validate the algorithm in real-world conditions, we conducted experiments using the Elephant Robotics 6-DOF myCobot280 arm equipped with an adaptive gripper across a series of pick-and-place tasks. The control policies were deployed on a Jetson Nano AI board, ensuring that evaluation was performed in a resource-constrained yet practically deployable setting. The goal of each experiment was to pick objects from their initial locations and place them at their designated target positions, as illustrated in Fig.~\ref{fig:real_robot}.  

Two experimental scenarios were designed, each progressively increasing in difficulty. In the first scenario, the robot was required to manipulate two objects placed at distinct locations, transporting each to its specified target position. In the second scenario, the task was extended to three objects placed at different locations, increasing both planning horizon and multimodal decision complexity. These tasks test the ability of the learned policy to generalize to multi-object, sequential pick-and-place settings.  

To further examine sample efficiency, the algorithm was trained and tested under two configurations: one using a dataset of $3000$ offline samples, and another using exactly half the data ($1500$ samples). In all settings, the robot successfully completed the required pick-and-place tasks, demonstrating that the learned policy transfers reliably to the physical robot even under reduced training data conditions.  

The main variation across dataset sizes was in object placement accuracy, measured by the Euclidean error between the end-effector (EE) and the target. With the full dataset ($3000$ samples), the average error stayed below $1.5$~cm in both scenarios. With the reduced dataset ($1500$ samples), all tasks were still completed, though errors rose to about $2.3$~cm in the two-object case and $2.8$~cm in the more challenging three-object case. This modest degradation shows the algorithm remains robust to reduced training data while retaining high task success.

These results highlight two key findings: the algorithm reliably scales to real multi-object manipulation, and it remains sample-efficient, achieving $100\%$ task success even with half the dataset, with only a slight increase in EE-to-target error.

\section{CONCLUSION}

This work introduced NF-HIQL, a normalizing flow–based extension of HIQL that employs expressive multimodal policies at both hierarchical levels. By replacing Gaussian policies with tractable flows, NF-HIQL enhances expressivity while maintaining stability through KL-divergence bounds and PAC-style sample efficiency guarantees. Experiments on OGBench show NF-HIQL consistently outperforms prior goal-conditioned and hierarchical baselines, including diffusion-based methods like BESO, especially in data-limited regimes. Notably, NF-HIQL trained with only 50\% of the data matches or surpasses full-dataset baselines across navigation, ball-dribbling, and manipulation tasks, demonstrating strong robustness and efficiency.

Beyond simulation, we validated NF-HIQL on a real-world robotic platform using the myCobot280 arm, where the policy reliably executed multi-object pick-and-place tasks. Even with limited offline data, the algorithm achieved $100\%$ task success with only modest increases in placement error, further underscoring its practical applicability in resource-constrained, real-world settings. 

Overall, these results establish NF-HIQL as a scalable, data-efficient approach to HGCRL. By combining the theoretical guarantees of HIQL with the representational power of normalizing flows, NF-HIQL provides a compelling framework for robust decision-making in both offline and real-world scenarios. Future directions include extending this framework to vision-based inputs, integrating with model-based components for planning, and applying it to more complex multi-agent and long-horizon robotic systems.

\addtolength{\textheight}{0cm}   

\bibliography{IEEEabrv,refs}

\begin{thebibliography}{10}
\providecommand{\url}[1]{#1}
\csname url@samestyle\endcsname
\providecommand{\newblock}{\relax}
\providecommand{\bibinfo}[2]{#2}
\providecommand{\BIBentrySTDinterwordspacing}{\spaceskip=0pt\relax}
\providecommand{\BIBentryALTinterwordstretchfactor}{4}
\providecommand{\BIBentryALTinterwordspacing}{\spaceskip=\fontdimen2\font plus
\BIBentryALTinterwordstretchfactor\fontdimen3\font minus \fontdimen4\font\relax}
\providecommand{\BIBforeignlanguage}[2]{{%
\expandafter\ifx\csname l@#1\endcsname\relax
\typeout{** WARNING: IEEEtran.bst: No hyphenation pattern has been}%
\typeout{** loaded for the language `#1'. Using the pattern for}%
\typeout{** the default language instead.}%
\else
\language=\csname l@#1\endcsname
\fi
#2}}
\providecommand{\BIBdecl}{\relax}
\BIBdecl

\bibitem{wu2024unifying}
C.~M. Wu, B.~Meder, and E.~Schulz, ``Unifying principles of generalization: past, present, and future,'' \emph{Annual Review of Psychology}, vol.~76, 2024.

\bibitem{margolis2024rapid}
G.~B. Margolis, G.~Yang, K.~Paigwar, T.~Chen, and P.~Agrawal, ``Rapid locomotion via reinforcement learning,'' \emph{The International Journal of Robotics Research}, vol.~43, no.~4, pp. 572--587, 2024.

\bibitem{lin2025sim}
T.~Lin, K.~Sachdev, L.~Fan, J.~Malik, and Y.~Zhu, ``Sim-to-real reinforcement learning for vision-based dexterous manipulation on humanoids,'' \emph{arXiv preprint arXiv:2502.20396}, 2025.

\bibitem{hao2024coordinated}
Z.~Hao, G.~Chen, Z.~Huang, Q.~Jia, Y.~Liu, and Z.~Yao, ``Coordinated transportation of dual-arm robot based on deep reinforcement learning,'' in \emph{2024 IEEE 19th Conference on Industrial Electronics and Applications (ICIEA)}.\hskip 1em plus 0.5em minus 0.4em\relax IEEE, 2024, pp. 1--6.

\bibitem{gong2024goal}
X.~Gong, F.~Dawei, K.~Xu, B.~Ding, and H.~Wang, ``Goal-conditioned on-policy reinforcement learning,'' \emph{Advances in neural information processing systems}, vol.~37, pp. 45\,975--46\,001, 2024.

\bibitem{li2022hierarchical}
J.~Li, C.~Tang, M.~Tomizuka, and W.~Zhan, ``Hierarchical planning through goal-conditioned offline reinforcement learning,'' \emph{IEEE Robotics and Automation Letters}, vol.~7, no.~4, pp. 10\,216--10\,223, 2022.

\bibitem{blonde2019sample}
L.~Blond{\'e} and A.~Kalousis, ``Sample-efficient imitation learning via generative adversarial nets,'' in \emph{The 22nd International Conference on Artificial Intelligence and Statistics}.\hskip 1em plus 0.5em minus 0.4em\relax PMLR, 2019, pp. 3138--3148.

\bibitem{mei2025deep}
S.~Mei and Y.~Wu, ``Deep networks as denoising algorithms: Sample-efficient learning of diffusion models in high-dimensional graphical models,'' \emph{IEEE Transactions on Information Theory}, 2025.

\bibitem{cao2024survey}
H.~Cao, C.~Tan, Z.~Gao, Y.~Xu, G.~Chen, P.-A. Heng, and S.~Z. Li, ``A survey on generative diffusion models,'' \emph{IEEE transactions on knowledge and data engineering}, vol.~36, no.~7, pp. 2814--2830, 2024.

\bibitem{xiong2024autoregressive}
J.~Xiong, G.~Liu, L.~Huang, C.~Wu, T.~Wu, Y.~Mu, Y.~Yao, H.~Shen, Z.~Wan, J.~Huang \emph{et~al.}, ``Autoregressive models in vision: A survey,'' \emph{arXiv preprint arXiv:2411.05902}, 2024.

\bibitem{choi2024data}
J.~Choi, S.~Byeon, and I.~Hwang, ``Data-driven closed-loop reachability analysis for nonlinear human-in-the-loop systems using gaussian mixture model,'' \emph{IEEE Transactions on Control Systems Technology}, 2024.

\bibitem{park2023hiql}
S.~Park, D.~Ghosh, B.~Eysenbach, and S.~Levine, ``Hiql: Offline goal-conditioned rl with latent states as actions,'' \emph{Advances in Neural Information Processing Systems}, vol.~36, pp. 34\,866--34\,891, 2023.

\bibitem{kobyzev2020normalizing}
I.~Kobyzev, S.~J. Prince, and M.~A. Brubaker, ``Normalizing flows: An introduction and review of current methods,'' \emph{IEEE transactions on pattern analysis and machine intelligence}, vol.~43, no.~11, pp. 3964--3979, 2020.

\bibitem{dinh2016density}
L.~Dinh, J.~Sohl-Dickstein, and S.~Bengio, ``Density estimation using real nvp,'' \emph{arXiv preprint arXiv:1605.08803}, 2016.

\bibitem{andrychowicz2017hindsight}
M.~Andrychowicz, F.~Wolski, A.~Ray, J.~Schneider, R.~Fong, P.~Welinder, B.~McGrew, J.~Tobin, O.~Pieter~Abbeel, and W.~Zaremba, ``Hindsight experience replay,'' \emph{Advances in neural information processing systems}, vol.~30, 2017.

\bibitem{yang2021density}
D.~Yang, H.~Zhang, X.~Lan, and J.~Ding, ``Density-based curriculum for multi-goal reinforcement learning with sparse rewards,'' \emph{arXiv preprint arXiv:2109.08903}, 2021.

\bibitem{nair2020goal}
S.~Nair, S.~Savarese, and C.~Finn, ``Goal-aware prediction: Learning to model what matters,'' in \emph{International Conference on Machine Learning}.\hskip 1em plus 0.5em minus 0.4em\relax PMLR, 2020, pp. 7207--7219.

\bibitem{charlesworth2020plangan}
H.~Charlesworth and G.~Montana, ``Plangan: Model-based planning with sparse rewards and multiple goals,'' \emph{Advances in Neural Information Processing Systems}, vol.~33, pp. 8532--8542, 2020.

\bibitem{zhu2021mapgo}
M.~Zhu, M.~Liu, J.~Shen, Z.~Zhang, S.~Chen, W.~Zhang, D.~Ye, Y.~Yu, Q.~Fu, and W.~Yang, ``Mapgo: Model-assisted policy optimization for goal-oriented tasks,'' \emph{arXiv preprint arXiv:2105.06350}, 2021.

\bibitem{kulkarni2016hierarchical}
T.~D. Kulkarni, K.~Narasimhan, A.~Saeedi, and J.~Tenenbaum, ``Hierarchical deep reinforcement learning: Integrating temporal abstraction and intrinsic motivation,'' \emph{Advances in neural information processing systems}, vol.~29, 2016.

\bibitem{nachum2018near}
O.~Nachum, S.~Gu, H.~Lee, and S.~Levine, ``Near-optimal representation learning for hierarchical reinforcement learning,'' \emph{arXiv preprint arXiv:1810.01257}, 2018.

\bibitem{robert2023sample}
A.~Robert, C.~Pike-Burke, and A.~A. Faisal, ``Sample complexity of goal-conditioned hierarchical reinforcement learning,'' \emph{Advances in Neural Information Processing Systems}, vol.~36, pp. 62\,696--62\,712, 2023.

\bibitem{jain2023learning}
V.~Jain and S.~Ravanbakhsh, ``Learning to reach goals via diffusion,'' \emph{arXiv preprint arXiv:2310.02505}, 2023.

\bibitem{madangoal2flownet}
K.~Madan, A.~Zhan, A.~Lamb, E.~Bengio, L.~Pan, G.~Berseth, and Y.~Bengio, ``Goal2flownet: Learning diverse policy covers using gflownets for goal-conditioned rl.''

\bibitem{ghugare2025normalizing}
R.~Ghugare and B.~Eysenbach, ``Normalizing flows are capable models for rl,'' \emph{arXiv preprint arXiv:2505.23527}, 2025.

\bibitem{mazoure2020leveraging}
B.~Mazoure, T.~Doan, A.~Durand, J.~Pineau, and R.~D. Hjelm, ``Leveraging exploration in off-policy algorithms via normalizing flows,'' in \emph{Conference on Robot Learning}.\hskip 1em plus 0.5em minus 0.4em\relax PMLR, 2020, pp. 430--444.

\bibitem{bellman1957markovian}
R.~Bellman, ``A markovian decision process,'' \emph{Journal of mathematics and mechanics}, pp. 679--684, 1957.

\bibitem{park2024ogbench}
S.~Park, K.~Frans, B.~Eysenbach, and S.~Levine, ``Ogbench: Benchmarking offline goal-conditioned rl,'' \emph{arXiv preprint arXiv:2410.20092}, 2024.

\bibitem{reuss2023goal}
M.~Reuss, M.~Li, X.~Jia, and R.~Lioutikov, ``Goal-conditioned imitation learning using score-based diffusion policies,'' \emph{arXiv preprint arXiv:2304.02532}, 2023.

\end{thebibliography}

\section*{APPENDIX}
\setcounter{lemma}{0}
\subsection{KL Bound for Advantage-Weighted RealNVP Policies}
\label{appdx:1}
\begin{lemma}[RealNVP Lower Bound]
\label{lem:realnvp-lower}
Fix $s\in\mathcal S$ and let $\pi_\theta(a\mid s)$ be a RealNVP density on $a$ with $L$ inverse coupling layers.
Therefore, $u=x^{(L)}=f_\theta^{-1}(a;s)$ and

\small 
\begin{equation}
\begin{split}
\log \pi_\theta(a\mid s)
&= \log p_z(u) 
+ \log\Big|\det \frac{\partial u}{\partial a}\Big|, \\ p_z(u) &= \mathcal N(0,I).
\end{split}
\end{equation}

\normalsize
Assume (i) bounded actions: $\|a\|_2 \le A_{\max}$; (ii) per-layer bounds: 
$\|s_\ell(\cdot)\|_\infty \le S_\ell$, $\|t_\ell(\cdot)\|_2 \le T_\ell$; and (iii) $d_\ell:=|I_\ell|$ scaled coordinates.
Then, for all such $a$,

\small
\begin{align}
\log \pi_\theta(a\mid s) 
&\ge -B, \\
B 
&:= \tfrac{d}{2}\log(2\pi) \;+\; \tfrac{1}{2}U_{\max}^2 \;+\; \sum_{\ell=1}^L d_\ell S_\ell, \\
U_{\max}
&:= \exp\!\Big(\sum_{\ell=1}^L S_\ell\Big)\Big(A_{\max} + \sum_{\ell=1}^L T_\ell\Big).
\end{align}
\end{lemma}

\normalsize
\begin{proof}
Let $x^{(0)}=a$ and $x^{(L)}=u$. For any diagonal matrix $D$,
$\|D v\|_2 \le \|D\|_{2}\,\|v\|_2$ and for $D=\mathrm{Diag}(e^{-s_\ell})$ we have 
$\|D\|_{2}=\max_j e^{-s_{\ell,j}} \le e^{S_\ell}$ by $\|s_\ell\|_\infty\le S_\ell$.
Thus, using the inverse update and the triangle inequality,
\begin{align}
\|x^{(\ell+1)}\|_2 &\le e^{S_\ell}\|x^{(\ell)}\|_2 + e^{S_\ell}T_\ell .
\end{align}
Iterating this inequality for $\ell=0,\dots,L-1$ gives
\small
\begin{dmath}
\|u\|_2 = \|x^{(L)}\|_2 
\le \exp\!\Big(\sum_{\ell=1}^L S_\ell\Big)
\Big(\|a\|_2 + \sum_{\ell=1}^L T_\ell\Big)
\le U_{\max}.
\end{dmath}
\normalsize
For the Gaussian base,
\small
\begin{equation}
-\log p_z(u) 
= \tfrac{d}{2}\log(2\pi) + \tfrac{1}{2}\|u\|_2^2
\le \tfrac{d}{2}\log(2\pi) + \tfrac{1}{2}U_{\max}^2 .
\end{equation}
\normalsize
The inverse Jacobian of a RealNVP layer is block-triangular with diagonal 
$\mathrm{Diag}(e^{-s_\ell})$ on the scaled block, hence

\small
\begin{equation}
\log\Big|\det \frac{\partial x^{(\ell+1)}}{\partial x^{(\ell)}}\Big|
= -\sum_{j\in I_\ell} s_{\ell,j}\big(x^{(\ell)}_{J_\ell}\big)
\ge - d_\ell S_\ell .
\end{equation}
\normalsize
Summing over layers yields

\small
\begin{equation}
\log\Big|\det \frac{\partial u}{\partial a}\Big|
= \sum_{\ell=0}^{L-1} 
\log\Big|\det \frac{\partial x^{(\ell+1)}}{\partial x^{(\ell)}}\Big|
\ge -\sum_{\ell=1}^L d_\ell S_\ell .
\end{equation}
\normalsize
Combining the base and Jacobian bounds gives
\small
\begin{align}
\log \pi_\theta(a\mid s)
&= \log p_z(u) + \log\Big|\det \frac{\partial u}{\partial a}\Big| \nonumber\\
&\ge -\Big[\tfrac{d}{2}\log(2\pi) + \tfrac{1}{2}U_{\max}^2\Big]
- \sum_{\ell=1}^L d_\ell S_\ell \\
&= -B .
\end{align}
\end{proof}
\normalsize
\begin{lemma}[KL Bound with Behavior Density Cap]
\label{lem:kl-bound}
Let $\pi^b(\cdot\mid s)$ be a behavior policy with $\pi^b(a\mid s)\le M<\infty$ for all $a,s$.
If $\pi_\theta(\cdot\mid s)$ satisfies Lemma~\ref{lem:realnvp-lower} with constant $B$, then for every $s$,

\small
\begin{align}
\mathrm{KL}\!\big(\pi^b(\cdot\mid s)\,\|\,\pi_\theta(\cdot\mid s)\big)
&= \underbrace{\int \pi^b(a\mid s)\big[-\log \pi_\theta(a\mid s)\big]\,da}_{\mathcal H(\pi^b,\pi_\theta)}
   \nonumber\\
&-
   \underbrace{\int \pi^b(a\mid s)\big[-\log \pi^b(a\mid s)\big]\,da}_{\mathcal H(\pi^b)} \nonumber\\[4pt]
&\le B + \log M .
\end{align}

\normalsize
\end{lemma}

\begin{proof}
By Lemma~\ref{lem:realnvp-lower}, $-\log \pi_\theta(a\mid s)\le B$ on the support of $\pi^b(\cdot\mid s)$, so
$\mathcal H(\pi^b,\pi_\theta)\le B$.
Since $\pi^b(a\mid s)\le M$ a.e., $\log \pi^b(a\mid s)\le \log M$ a.e., hence

\small
\begin{dmath}
\mathcal H(\pi^b)
= -\int \pi^b \log \pi^b \, da
\ge -\int \pi^b \log M \, da
= -\log M .
\end{dmath}

\normalsize
Therefore $\mathrm{KL}(\pi^b\|\pi_\theta)\le B - (-\log M) = B + \log M$.
\end{proof}

\subsection{Sample Efficiency for HIQL with Flow Policies}
\label{appendix:sample}
For each level $L \in \{h,\ell\}$, let $V$ be a shared value function used to compute per-level advantages $A_L(a,s)$. Define the advantage-weights as $w_L(a,s) = e^{\beta A_L(a,s)}$. Given an offline dataset inducing a distribution $d_L^\mu(s,a)$ at each level, the training objective is $\mathcal J_L(\pi) = \E_{(s,a)\sim d_L^\mu}\!\big[w_L(a,s)\,\log \pi(a\mid s)\big]$, under the following assumptions.
\begin{itemize}
    \item[(A1)] \textbf{Bounded rewards:} $|R(s,g)|\le R_{\max}$ and $\gamma\in(0,1)$.
    \item[(A2)] \textbf{Concentrability (coverage):} For each level $L$, the reference state occupancy $d_L^\star$ is continuous w.r.t.\ the dataset occupancy $d_L^\mu$, with density ratio bounded by $C_L\ge1$:
\small
\begin{equation}
\sup_{s}\frac{d_L^\star(s)}{d_L^\mu(s)}\le C_L.
\end{equation}
\normalsize
    \item[(A3)] \textbf{Weight control:} Weights are bounded, either because of clipping:
\(
0\le w_L(a,s)\le W_{\max}
\) or $A_L$ is bounded.
    \item[(A4)] \textbf{Policy class capacity \& bounded loss:} For each $L$, the log-density class $\mathcal F_L=\{(s,a)\mapsto \log \pi(a\mid s):\pi\in\Pi_L\}$ has finite Rademacher complexity $\mathfrak R_{n_L}(\mathcal F_L)$, and $\big|\log \pi(a\mid s)\big|\le B_L$. A bound on per-level \emph{environment} advantages is also assumed: $|A^{\pi_L}(s,a)|\le A_{\max,L}.$

\end{itemize}

\subsubsection{KL reduction.}
For each $s$, define
\small
\begin{equation}
\begin{split}
Z_L(s) &= \int w_L(a,s)\,p_{\text{data}}(a\mid s)\,da,
\\
q_L(a\mid s) &= \frac{w_L(a,s)\,p_{\text{data}}(a\mid s)}{Z_L(s)}.
\end{split}
\end{equation}
\normalsize
Then
\small
\begin{equation}
\begin{split}
\mathcal J_L(\pi\mid s)
&= \int w_L(a,s)\,p_{\text{data}}(a\mid s)\,\log \pi(a\mid s)\,da \\
&= Z_L(s)\,\E_{q_L}[\log \pi].
\end{split}
\end{equation}
\normalsize
Adding and subtracting $\E_{q_L}\log q_L$ gives
\small
\begin{equation}
\mathcal J_L(\pi\mid s) = Z_L(s)\,\Big(\E_{q_L}\log q_L - \KL(q_L\|\pi)\Big).
\end{equation}
\normalsize
Thus
\small
\begin{equation}
\sup_{\pi'}\mathcal J_L(\pi'\mid s)-\mathcal J_L(\pi\mid s)
= Z_L(s)\,\KL(q_L\|\pi).
\end{equation}
\normalsize
Averaging over $s\sim d_L^\mu$,
\small
\begin{equation}
\sup_{\pi'}\mathcal J_L(\pi')-\mathcal J_L(\pi)
= \E_{s\sim d_L^\mu}\!\big[\,Z_L(s)\,\KL(q_L\|\pi)\,\big].
\end{equation}
\normalsize
\subsubsection{Weighted ERM Generalization Bounds}

Let $\ell_\pi(a,s)=-\log\pi(a\mid s)\in[0,B_L]$ with $0\le w\le W_{\max}$. Define
\small
\begin{equation}
L(\pi):=\E[w\,\ell_\pi], \qquad
\widehat L_n(\pi):=\frac1n\sum_{i=1}^{n} w_i\,\ell_\pi(a_i,s_i).
\end{equation}
\normalsize
By symmetrization and Hoeffding’s inequality, with probability $\ge 1-\delta$,
\small
\begin{equation}
\sup_{\pi\in\Pi_L} \big|L(\pi)-\widehat L_n(\pi)\big|
\le
2\,\mathfrak R_n(\mathcal G_L) + W_{\max}B_L\sqrt{\tfrac{\log(1/\delta)}{2n}},
\end{equation}
\normalsize
where $\mathcal G_L=\{(s,a)\mapsto w(s,a)\,\ell_\pi(s,a):\pi\in\Pi_L\}$ and
\small
\begin{equation}
\mathfrak R_n(\mathcal G_L)
:= \E_{\bm z,\bm\sigma}\!\left[\sup_{\pi\in\Pi_L}\frac1n\sum_{i=1}^n \sigma_i\, w_i\,\ell_\pi(a_i,s_i)\right],
\end{equation}
\normalsize
with $\sigma_i$ i.i.d.\ Rademacher variables.

\subsubsection{Multiplier contraction (bounded weights).}
For $\mathcal H=\{\ell_\pi:\pi\in\Pi_L\}$ and $\mathcal G_L=\{w\,\ell_\pi:\ell_\pi\in\mathcal H\}$ with $0\le w\le W_{\max}$, the Ledoux–Talagrand contraction inequality with $\phi_i(t)=\alpha_i t$, $\alpha_i=w_i/W_{\max}\in[0,1]$ and for fixed sample $(z_i)=(s_i,a_i)$, write
\small
\begin{equation}
\begin{split}
    \mathfrak R_n(\mathcal G_L\mid \bm z)
&= \E_{\bm\sigma}\left[\sup_{\pi}\frac1n\sum_{i=1}^n \sigma_i w_i\,\ell_\pi(z_i)\right] \\
&= W_{\max}\,\E_{\bm\sigma}\left[\sup_{\pi}\frac1n\sum_{i=1}^n \sigma_i \alpha_i\,\ell_\pi(z_i)\right],
\end{split}
\end{equation}
\normalsize
The contraction inequality gives - 
\small
\begin{equation}
\begin{split}
    \E_{\bm\sigma}\left[\sup_{\pi}\frac1n\sum_{i=1}^n \sigma_i \phi_i\big(\ell_\pi(z_i)\big)\right]
&\le
\E_{\bm\sigma}\left[\sup_{\pi}\frac1n\sum_{i=1}^n \sigma_i \ell_\pi(z_i)\right].
\end{split}
\end{equation}
\normalsize
Taking expectation over $\bm z$ yields $\mathfrak R_n(\mathcal G_L)\le W_{\max}\,\mathfrak R_n(\mathcal H)$. Therefore, with probability $\ge 1-\delta$,
\small
\begin{dmath}
\sup_{\pi\in\Pi_L}\big|L(\pi)-\widehat L_n(\pi)\big|
\le
2W_{\max}\,\mathfrak R_n(\mathcal F_L)+ \\
W_{\max}B_L\sqrt{\tfrac{\log(1/\delta)}{2n}}.
\end{dmath}
\normalsize

\subsubsection{ERM $\Rightarrow$ KL control.}
\label{sec:erm}
Let $\hat\pi_L=\argmax_{\pi\in\Pi_L}\widehat{\mathcal J}_L(\pi)$. Since
\small
\begin{equation}
\widehat{\mathcal J}_L(q_L) - \widehat{\mathcal J}_L(\hat\pi_L)\le 0,
\end{equation}
\normalsize
we obtain
\small
\begin{equation}
\mathcal J_L(q_L)-\mathcal J_L(\hat\pi_L)
\le 2\sup_{\pi}\big|\mathcal J_L(\pi)-\widehat{\mathcal J}_L(\pi)\big|.
\end{equation}
\normalsize
From the KL reduction step,
\small
\begin{equation}
\mathcal J_L(q_L)-\mathcal J_L(\hat\pi_L)
= \E_x \Big[ Z_L(s)\,\KL(q_L\|\hat\pi_L)\Big].
\end{equation}
\normalsize
Using $Z_L(s)\le W_{\max}$ and the ERM bound,
\small
\begin{dmath}
\E_{s}\KL(q_L(\cdot\mid s)\,\|\,\hat\pi_L(\cdot\mid s))
\le
4\,\mathfrak R_n(\mathcal F_L)
+2B_L\sqrt{\tfrac{\log(1/\delta)}{2n}}.
\end{dmath}
\normalsize

\subsubsection{Return gap from KL.}
For any level $L$ with effective timescale $H_L$ and per-level policies $\pi^\star_L,\pi_L$, the performance-difference lemma gives
\label{sec:kl_return}
\small
\begin{equation}
J(\pi^\star_L)-J(\pi_L)
= \frac{1}{1-\gamma}\,\E_{s\sim d^{\pi^\star_L},a\sim\pi^\star_L}[A^{\pi_L}(s,a)].
\end{equation}
\normalsize
Since $|A^{\pi_L}|\le A_{\max,L}$ and $\TV(p,q)\le\sqrt{\tfrac12\,\KL(p\|q)}$,
\small
\begin{equation}
\E_{a\sim\pi^\star_L}[A^{\pi_L}(s,a)]
\le A_{\max,L}\sqrt{\tfrac12\,\KL(\pi^\star_L\|\pi_L)}.
\end{equation}
\normalsize
Applying Jensen to $\sqrt{\cdot}$ and concentrability,
\small
\begin{equation}
J(\pi^\star_L)-J(\pi_L)
\le \frac{H_L A_{\max,L}}{1-\gamma}\;
\sqrt{\frac{C_L}{2}\,\E_{s\sim d_L^\mu}\KL(\pi^\star_L\|\pi_L)}.
\end{equation}
\normalsize

\subsubsection{Main Theorem}

\begin{lemma}[PAC-style sample efficiency with explicit constants]
\label{thm:pac-explicit}
Let $\hat\pi_h,\hat\pi_\ell$ be the learned per-level policies via advantage-weighted MLE with conditional flows.
Under Assumptions 1--4, with probability at least $1-\delta$,
\small
\begin{dmath}
J(\pi^\star)-J(\hat\pi_{h,\ell})
\le
\frac{H_h A_{\max,h}}{1-\gamma}\,
\sqrt{\frac{C_h}{2}\;
\Big(4\,\mathfrak R_{n_h}(\mathcal F_h)+2B_h\sqrt{\tfrac{\log(2/\delta)}{2n_h}}\Big)}
\\ +
\frac{H_\ell A_{\max,\ell}}{1-\gamma}\,
\sqrt{\frac{C_\ell}{2}\;
\Big(4\,\mathfrak R_{n_\ell}(\mathcal F_\ell)+2B_\ell\sqrt{\tfrac{\log(2/\delta)}{2n_\ell}}\Big)}
\\ + 
\varepsilon_V,
\end{dmath}
\normalsize
where $\varepsilon_V$ is the uniform value-estimation error used to compute $A_L$.
\end{lemma}

\begin{proof}
Apply Theory in \ref{sec:erm} to bound $\E\KL(q_L\|\hat\pi_L)$, then theory in  \ref{sec:kl_return} with $\pi^\star_L:=q_L$ and $\pi_L:=\hat\pi_L$ for each level, and sum the two contributions. Add $\varepsilon_V$ for the value approximation.
\end{proof}

\end{document}